\documentclass[twocolumn]{autart}

\usepackage{amsmath}
\usepackage{amsfonts}

\newcommand{\tauE}[2]{\underset{\tau \sim #1}{\mathbb{E}}\left[ #2 \right]}

\usepackage{booktabs}
\usepackage{amssymb}
\usepackage{pifont}%
\usepackage{graphicx}
\usepackage{subfigure}
\usepackage{verbatim}
\usepackage{dsfont}
\usepackage{mathrsfs}
\usepackage{times}
\usepackage{algorithm}

\usepackage{algorithmic} 
\usepackage[hyphens]{url}
\usepackage{array}
\usepackage{color}

\usepackage{nicefrac}       
\usepackage[compress]{cite}

\newtheorem{theorem}{Theorem}
\newtheorem{definition}{Definition}
\newtheorem{lemma}{Lemma}
\newtheorem{assumption}{Assumption}

\newtheorem{remark}{Remark}
\newenvironment{proof}{{\noindent\it Proof}\quad}{\hfill $\square$\par}

\begin{document}
\begin{frontmatter}

\title{Reinforcement Learning for Control with Probabilistic Stability Guarantee: A Finite-Sample Approach
}

\author[lxz-a]{Minghao Han},  
\author[lxz-b]{Lixian Zhang}, 
\author[cll-a]{Chenliang Liu}, 
\author[wp-b]{Zhipeng Zhou}, 
\author[jw-c]{Jun Wang}, 
\author[wp-c]{Wei Pan}
\address[lxz-a]{Department of Control Science and Engineering, Harbin Institute of Technology, China.}
\address[lxz-b]{State Key Laboratory of Robotics 
	and Systems (HIT), Harbin Institute of Technology, China.}
\address[cll-a]{School of Automation, Central South University, China.}
\address[wp-b]{Department of Cognitive Robotics, Delft University of Technology, Netherlands.}
\address[jw-c]{Department of Computer Science, University College London, UK.}
\address[wp-c]{Department of Computer Science, University of Manchester, UK.}
\thanks{{This work was supported in part by National Natural Science Foundation of China (62225305,62527807,62503507), in part by the State Key Laboratory of Robotics and Systems (HIT) under Grant SKLRS-2025-KF-05.
}}
\thanks{Corresponding author: Lixian Zhang (Email: lixianzhang@hit.edu.cn).}


\begin{abstract}
This paper presents a novel approach to reinforcement learning (RL) for control systems that provides probabilistic stability guarantees using finite data. Leveraging Lyapunov's method, we propose a probabilistic stability theorem that ensures mean square stability using only a finite number of sampled trajectories. The probability of stability increases with the number and length of trajectories, converging to certainty as data size grows. Additionally, we derive a policy gradient theorem for stabilizing policy learning and develop an RL algorithm, L-REINFORCE, that extends the classical REINFORCE algorithm to stabilization problems. The effectiveness of L-REINFORCE is demonstrated through simulations on a Cartpole task, where it outperforms the baseline in ensuring stability. This work bridges a critical gap between RL and control theory, enabling stability analysis and controller design in a model-free framework with finite data.

\end{abstract}
\end{frontmatter}


\section{Introduction}

Reinforcement learning (RL) has achieved superior performance on some complicated control tasks that are challenging for traditional control engineering methods \cite{kaufmann2023champion, xia2025data, shi2025reinforcement}.
An optimal controller can be learned through trial and error  \cite{zhao2023linear}. However, without using a mathematical model of the system, {securing the stability of the closed-loop system still remains a challenging issue for the sample-based RL methods}.

The most effective and general approach for the stability analysis of a dynamical system is Lyapunov's method \cite{lyapunov1892general}. 
As a basic tool in control theory, the construction/learning of the Lyapunov function is not trivial, and many works are devoted to this problem. 
In~\cite{perkins2002lyapunov}, the RL agent controlled the switch between designed controllers using Lyapunov domain knowledge so that any policy is safe and reliable. \cite{petridis2006construction} proposes a straightforward approach for constructing the Lyapunov functions for nonlinear systems using neural networks. \cite{richards2018lyapunov} proposed a learning-based approach for constructing Lyapunov neural networks with the maximized region of attraction. However, these approaches explicitly require a model of the system dynamics. 
In the well-established control theory, the mathematical model of the dynamics plays an essential role in stability analysis and controller design.  However, in model-free learning methods, as the dynamic model is unknown, the Lyapunov condition has to be verified by trying out \emph{all} possible consecutive data pairs in the state space, which makes it impractical to directly exploit Lyapunov's method in a model-free framework. In \cite{berkenkamp2017safe}, local stability was analyzed by validating the ``energy decreasing'' condition on discretized points in the subset of state space with the help of a learned model, meaning that only a finite number of inequalities need to be checked. 
Nevertheless, the discretization technique may become infeasible as the dimension and space of interest increase, limiting its application to rather simple and low-dimensional systems. 

More recently, the concept of control with stability guarantee has been explored in the context of model-free RL. 
In~\cite{chow2018lyapunov}, a Lyapunov-based approach for solving constrained Markov decision processes was proposed with a novel way of constructing the Lyapunov function through linear programming. 
It should be noted that even though Lyapunov-based methods were adopted, the stability of the system was not addressed. 
In~\cite{postoyan2017stability}, an initial result was proposed for the stability analysis of deterministic nonlinear systems with an optimal controller for infinite-horizon discounted cost, based on the assumption that the discount is sufficiently close to $1$. However, in practice, it is rather difficult to guarantee the optimality of the learned policy unless certain assumptions on the system dynamics are made \cite{jiang2015global}. In \cite{huang2024specified}, a data-driven output tracking design for unknown linear discrete-time systems was investigated to achieve zero tracking error and user-specified convergence rate. It was shown in \cite{han2020actor} that the stability of Markov decision process (MDP) can be analyzed in a model-free and data-driven manner with the Lyapunov method, which is further incorporated with RL for stabilization tasks. {\cite{chang2019neural} proposed an algorithmic framework to learn control policies and neural network Lyapunov functions for the general class of deterministic nonlinear systems. \cite{chang2021stabilizing} further integrated the neural Lyapunov method proposed in \cite{chang2019neural} with model-free RL to learn stabilizing controllers based on data. \cite{ganai2023learning} proposed a novel model-free method to accelerate Learning from Observations based on the framework in \cite{chang2019neural} for the stabilization control problems in robotics. \cite{dawson2022safe} developed a model-based learning approach to synthesize robust feedback controllers based on a robust extension to the theory of control Lyapunov barrier functions. These papers represent an important stream of work in improving the reliability of learning-based control for complex nonlinear systems.}
However, the results above relied on an infinite amount of data to establish the stability guarantee, which is not practical in reality. As a result, how to analyze stability based on a finite amount of data remains an open question.

{
On the topic of finite sample-based RL for MDPs, many efforts have been made, and significant progress has been achieved. \cite{kearns1998finite} investigated the convergence rate of Q learning on MDPs with respect to the number of samples. \cite{lazaric2012finite} derived the performance bound and conducted the finite-sample analysis for least-squares temporal-difference learning of control policies. \cite{yang2018finite} analyzed the finite-sample performance of the batch actor-critic algorithm. The finite sample analysis of many more classical RL algorithms is referred to \cite{zou2019finite,chen2021finite,zhang2021sample,zhang2021finite}. Nonetheless, these results typically characterize the performance of the control policy in terms of the discounted sum of rewards given finite samples, while the notion of stability of the closed-loop system is missing and unaddressed. 

Another line of work focuses on the finite sample analysis of stochastic approximation (SA) \cite{robbins1951stochastic}, which forms the basis of many machine learning problems. 
\cite{chen2022finite} studied a nonlinear stochastic approximation algorithm under Markovian noise and established its finite-sample convergence bounds under various step sizes. The theoretical results were further used to derive the finite-sample bounds
of the Q-learning algorithm with linear function approximation \cite{chen2022finite}. A rich literature can be found on finite-sample analysis of the SA algorithm with application to RL of MDPs \cite{dalal2018finite,chen2020finite,huang2020stochastic}. However, these works focused on the convergence of the learning process with respect to the number of samples, and the convergence of the states was not considered.
The finite-sample analysis for the identification of nonlinear stochastic systems is also of interest for statistical learning \cite{tsiamis2019finite,tsiamis2023statistical}, though the topic of stability of the system was not a focus. 
The stability of Markov chains, along with the necessary and sufficient conditions, has been long investigated \cite{meyn2012markov,khasminskii2012stochastic}. Nevertheless, a finite-sample-based stability analysis for the MDPs is still missing. 
}

In this paper, we show that the mean square stability of the system can be analyzed based on a finite number of samples without knowing the model of the system. { Furthermore, we show how the analysis enables the development of an RL algorithm that learns a stabilizing control policy.} The contributions of this paper are summarized as follows:
\begin{enumerate}
    \item Instead of using infinite sample pairs, given a number of $M$ trajectories of $T$ steps in length, a stability theorem is proposed to provide a probabilistic stability guarantee for the system. The probability is an increasing function of the number $M$ and length $T$ of sampled trajectories, and it converges to $1$ as $M$ and $T$ grow. 
    \item As an independent interest, we also derive the policy gradient theorem for learning a stabilizing policy with sample pairs and the corresponding algorithm. We further reveal that the classic REINFORCE algorithm \cite{williams1992simple} is a special case of the proposed algorithm for the stabilization problem.
\end{enumerate}

The paper is organized as follows: In Section~\ref{sec:3}, the problem statement and the preliminary results are given. In Section~\ref{sec:5}, we propose the probabilistic stability guarantee when only a finite number of samples are accessible and the probabilistic bound is derived. In Section~\ref{sec:6}, based on the theoretical results, the policy gradient is derived and a model-free RL algorithm (L-REINFORCE) is given. 
In Section~\ref{sec:experiment}, the vanilla version of L-REINFORCE is tested on a simulated Cartpole stabilization task to demonstrate its effectiveness.

\section{Preliminaries} \label{sec:3}

\subsection{Problem statement}\label{sec:3-1}

In this work, we focus on the control of stochastic and nonlinear systems characterized by the following Markov decision process (MDP), 
\begin{equation}
    s_{t+1} \sim P(s_{t+1}|s_{t},a_{t}), \forall t\in \mathbb{Z}_+ \label{intro:system}
\end{equation}
where $s\in \mathcal{S}\subseteq \mathbb{R}^n$ denotes the state and $a\in \mathcal{A}$ denotes the action. {The state space $\mathcal{S}$ is continuous and the action space $\mathcal{A}$ is finite.} $P(s_{t+1}|s_{t},a_{t})$ denotes the state transition probability density function.  {The control policy $\pi(a|s)$ is the probability of taking the action $a\in\mathcal{A}$ given the state $s\in\mathcal{S}$.} {The distribution of the initial state $s_1\in\mathcal{S}$, is characterized by the probability density function $\rho (s)$.}

The mean square stability (MSS) of the stochastic systems described in \eqref{intro:system} is defined as follows.
\begin{definition}\cite{bolzern2010markov}
\label{def:mss}
The stochastic system \eqref{intro:system} is said to be mean square stable (MSS) if there exists a positive constant $b$ such that $\lim_{t\rightarrow \infty }\mathbb{E}_{s_{t}} \Vert s_{t}\Vert^2_2=0$ holds for any initial condition $s_{1}\in \{s_{1}|\Vert s_{1}\Vert^2_2\leq b\}$. If $b$ is arbitrarily large then the stochastic system is globally mean square stable (GMSS). 
\end{definition}

{The aim of this paper is twofold. First, we aim to conduct a finite sample stability analysis for the system \eqref{intro:system} and derive the probabilistic lower bound.} This probabilistic lower bound goes to 1 as the amount of data approaches infinity. 

{Second, we aim to propose an RL algorithm to learn a stabilizing control policy. The goal of the RL algorithm is to find a control policy $\pi$ such that the system in (1) is MSS with a probability lower bounded by $\delta$.
The mathematical formulation is given as follows,
\begin{align}
    \text{find }& \pi, \text{ s.t. } \mathbb{P}\left(\lim_{t\rightarrow \infty }\mathbb{E}_{s_{t}} \Vert s_{t}\Vert^2_2=0\right) \geq\delta
\end{align}}

Before proceeding, some notations are to be defined. We introduce $c(s):= \min(\Vert s\Vert_2^2, \overline{c}), \overline{c}>0$ to denote the clipped norm of state. The closed-loop transition probability density function is denoted as $P_\pi(s'|s):=\mathbb{E}_{a\sim\pi}P(s'|s,a)$. We also introduce the closed-loop state distribution at a certain instant $t$ as $P(s|\rho, \pi, t)$, which could be defined iteratively: $P(s'|\rho, \pi, t+1) =\int_\mathcal{S} P_\pi(s'|s)P(s|\rho, \pi, t)\mathrm{d}s, \forall t\in\mathbb{Z}_{[1,\infty)}$ and $P(s|\rho, \pi, 1) =\rho(s)$.

\subsection{Sample-based stability analysis}
\label{sec:3-1-2}

In this subsection, we will introduce the results of  stability analysis based on an infinite amount of data. First, the following assumption is commonly exploited in RL literature \cite{sutton2009convergent}.

\begin{assumption}\label{assumption: stationary distribution assumption}
There exists a unique stationary distribution $q_\pi(s)=\lim_{t\rightarrow\infty}P(s|\rho,\pi,t)$, { for the MDP in \eqref{intro:system} under the policy $\pi$ and initial state distribution $\rho$}.
\end{assumption}

{A further assumption on the initial state distributions is needed for the sample-based stability analysis theorem \cite{han2020actor}.}

\begin{assumption}\label{initial state assumption}
There exists a positive constant $b$ such that for any $s\in\{s|c(s)\leq b\}$, {the initial state probability density function $\rho(s)> 0$}.
\end{assumption}

{
\begin{remark}
    Assumption~\ref{assumption: stationary distribution assumption} essentially assumes that there exists a stationary state distribution $q_\pi$, which is commonly assumed in the RL literature \cite{zou2019finite,chen2021finite,zhang2021sample,zhang2021finite,sutton2009convergent}. In this work, the assumption is made on the pair consisting of the MDP and the policy $\pi$. We do not require that Assumption~\ref{assumption: stationary distribution assumption} hold for any control policy. Assumption~\ref{initial state assumption} requires that there exists a sublevel set specified by $b$, where the probability density function of the initial state distribution $\rho (s)>0$ for any state. According to \cite{han2020actor}, this assumption is to ensure that the states in the region of attraction are accessible in a data-based stability theorem. If $\rho(s)>0$ for any $s\in\mathcal{S}$, then $b$ can be an arbitrary positive value.
\end{remark}
}

Based on the above assumptions, one can exploit Lyapunov's method to prove the sample-based stability theorem. 
In this paper, we construct the Lyapunov function using the following parameterization,
\begin{equation}
    L(s) = (f_\phi(s) - f_\phi(0))^2 + \sigma c(s) \label{eq:lyapunov network}
\end{equation}
where $f_\phi(s)$ is a fully connected neural network (NN) with ReLU activation function. $\phi$ denotes the parameters of the network and $\sigma$ is a small positive constant. {We introduce the following sample-based stability analysis lemma from \cite{han2020actor}.}

\begin{lemma}\cite{han2020actor}\label{them:MSS}
{Assume that Assumptions~\ref{assumption: stationary distribution assumption}-\ref{initial state assumption} hold for system \eqref{intro:system}.} The stochastic system (\ref{intro:system}) is mean square stable if there exists a function $L:\mathcal{S}\rightarrow \mathbb{R}_{+}$\ and positive constants $\alpha _{1}$, $\alpha _{2}$ and $\alpha_{3}$, such that%
\begin{align}
\begin{split}\label{Theorem 2-1}
\ \ \ \ \ \ \ \ \ \ \ \ \ \ \ \alpha_{1}c\left( s\right) \leq L(s)\leq \alpha _{2}c\left( s\right)
\end{split}\\
\begin{split}\label{Theorem 2-2}
\mathbb{E}_{s\sim \mu_\pi }(\mathbb{E}_{s^{\prime }\sim P_{\pi }}L(s^{\prime
})-L(s) + \alpha_{3}c\left( s\right))\leq 0
\end{split}
\end{align}
where 
\begin{equation}
\mu_\pi(s):= \lim_{T\rightarrow\infty}\frac{1}{T}\sum_{t=1}^{T} P(s|\rho,\pi,t)
\end{equation}
is the infinite sampling distribution. 
\end{lemma}

It can be found that in {Lemma}~\ref{them:MSS}, the infinite number of energy-decreasing conditions are replaced by only a single sample-based inequality (\ref{Theorem 2-2}). However, the validation of stability through a sample-based approach comes with a cost: it theoretically requires a tremendous, if not infinite, number of samples to thoroughly estimate the state distributions at instants from $0$ to infinity, which is impractical.

\section{Finite-sample stability analysis}
\label{sec:5}

In this section, we will show that a finite number of samples should be informative enough to guarantee stability with a certain probability. More specifically, the probabilistic stability bound will be given by closing the gap between infinite and finite-sample guarantees.

To estimate the $\mu_\pi$ in Lemma~\ref{them:MSS}, an infinite number of trajectories of infinite time steps are needed, whereas in practice only $M$ trajectories of $T$ time steps are accessible. Thus, in this section, we will first introduce the finite-time sampling distribution (FSD) $\mu_\pi^T := \frac{1}{T}\sum_{t=1}^{T} P(s|\rho,\pi,t)$, as an intermediate to study the effect of the finite sample-based estimation.
Apparently, $\lim_{T\rightarrow\infty}\mu_\pi^T=\mu_\pi$.

The general idea of exploiting $\mu_\pi^T$ is: we first derive the  deviation of $\mathbb{E}_{\mu_\pi^T}\Delta L(s)$ from $\mathbb{E}_{\mu_\pi }\Delta L(s)$ with respect to $T$, where 
\begin{equation}
    \Delta L(s) := \mathbb{E}_{s'\sim P_\pi}L(s')-L(s)+\alpha_3c(s)
\end{equation}
then we study the effect of estimating $\mathbb{E}_{\mu_\pi^T}\Delta L(s)$ with sample average and derive the probabilistic bound. Finally, the above effects are unified to propose the finite sample-based stability guarantee.

Now, we first close the first gap in terms of deviation between $\mathbb{E}_{\mu_\pi^T}\Delta L(s)$ and $\mathbb{E}_{\mu_\pi }\Delta L(s)$. To quantitatively analyze this effect, we introduce the following assumption.

\begin{assumption}\label{assumption:uniform ergodic}
 There exists a constant $\gamma\in(0,1)$ such that for any $\pi$
\begin{equation}
    \sum\limits_{t=1}^{T}\|P(s|\pi,\rho,t)-q_{\pi}(s)\|_{1}\leq 2 T^{\gamma},\forall ~ T\in\mathbb{Z}_+
\end{equation}
where $\|P(s|\pi,\rho,t)-q_{\pi}(s)\|_{1}$ denotes the $L_1$-distance between probability measures $P$ and $q_\pi$.
\end{assumption}

\begin{remark}
It should be noted that the Assumption~\ref{assumption:uniform ergodic} can be satisfied for ergodic MDPs. $q_\pi$ denotes the stationary state distribution, it naturally follows that $\sum\limits_{t=1}^{T}\Vert P(s|\pi,\rho,t)-q_{\pi} (s)\Vert_{1}\leq2 T^{\gamma(T)}\leq2 T$, where $\gamma(T)\in [0,1]$ without any further assumption. The assumption proposed replaces this time-varying $\gamma(T)$ with a constant. 
{Many existing methods assume that the state distribution converges to $q_\pi$ exponentially with respect to time \cite{chen2022finite,zou2019finite}, and this assumption holds for irreducible and aperiodic Markov chains \cite{levin2017markov}. In our derivation, we do not need the state distribution to uniformly converge to $q_\pi$ at exponential speed, and therefore relaxed the assumption on the convergence rate. }
Nevertheless, Assumption~\ref{assumption:uniform ergodic} allows us to give the quantitative bound for the deviation between $\mathbb{E}_{\mu_\pi }\Delta L(s)$ and $\mathbb{E}_{\mu_\pi^T}\Delta L(s)$ with respect to $T$. 
\end{remark}

Based on Assumption~\ref{assumption:uniform ergodic}, we introduce the following Lemma.
\begin{lemma}\label{lemma:finite time estimation bound}
{Assume that Assumptions~\ref{assumption: stationary distribution assumption} and \ref{assumption:uniform ergodic} hold for system \eqref{intro:system}.} If there exist positive constants $\alpha _{1}$, $\alpha _{2}$ such that (\ref{Theorem 2-1}) hold, then
\begin{equation}\label{eq:delta L bound}
    \left|\mathbb{E}_{\mu_\pi}\Delta L(s)-\mathbb{E}_{\mu_\pi^T}\Delta L(s)\right|
    \leq  2\overline{c}(\alpha_3+\alpha_2)T^{\gamma-1}
\end{equation}

\end{lemma}

\begin{proof}
First of all, we can quantify the variation of  $\mathbb{E}_{\mu_\pi}\Delta L(s) -\mathbb{E}_{\mu_\pi^T}\Delta L(s)$  with respect to $T$.
\begin{align}
    & \mathbb{E}_{\mu_\pi}\Delta L(s)-\mathbb{E}_{\mu_\pi^T}\Delta L(s)\notag\\
    =&\int_{\mathcal{S}}\left(\mu_\pi(s)-\frac{1}{T}\sum_{t=1}^T P(s|\rho,\pi,t)\right)\Delta L(s)\text{d}s\\
    =&\frac{1}{T}\sum_{t=1}^T\int_{\mathcal{S}}\left(\mu_\pi(s)- P(s|\rho,\pi,t)\right)\Delta L(s)\text{d}s
\end{align}

Therefore, it follows that 
\begin{equation}
\begin{aligned}
    &\left|\mathbb{E}_{\mu_\pi}\Delta L(s)-\mathbb{E}_{\mu_\pi^T}\Delta L(s)\right| \\
    \leq &\frac{1}{T}\sum_{t=1}^T\left \Vert\mu_\pi(s)- P(s|\rho,\pi,t)\right\Vert_1\left \Vert\Delta L(s)\right\Vert_\infty
\end{aligned}
\end{equation}
where $\left \Vert\mu_\pi(s)- P(s|\rho,\pi,t)\right\Vert_1$ is $L_1$-distance between $\mu_\pi$ and $P(s|\rho,\pi,t)$, which is bounded by $2$. $\left \Vert\Delta L(s)\right\Vert_\infty $ is the upper bound of $\Delta L$, according to (\ref{Theorem 2-1}), $$\Delta L(s) =  \mathbb{E}_{s'\sim P_\pi}L(s')-L(s)+\alpha_3c(s) \leq \alpha_2 \overline{c} - 0 + \alpha_3 \overline{c}$$
Thus we have
\begin{equation}
    \left \Vert\Delta L(s)\right\Vert_\infty \leq (\alpha_3+\alpha_2)\overline{c}
\end{equation}
{According to Assumption~\ref{assumption: stationary distribution assumption}, the sequence $\{P(s|\rho,\pi,t), t\in\mathbb{Z}_+\}$ converges to $q_\pi(s)$ as $t$ approaches $\infty$, then by the Abelian theorem, the sequence $\{\frac{1}{T}\sum_{t=1}^T P(s|\rho,\pi,t), T\in\mathbb{Z}_+\}$ also converges to $q_\pi(s)$. Therefore, $\mu_\pi(s) = q_\pi(s)$.} At last, one obtains the bound in (\ref{eq:delta L bound}),
\begin{align}
    &\left|\mathbb{E}_{\mu_\pi}\Delta L(s)-\mathbb{E}_{\mu_\pi^T}\Delta L(s)\right|\notag\\
    \leq & \frac{1}{T}\sum_{t=1}^T\left \Vert\mu_\pi(s)- P(s|\rho,\pi,t)\right\Vert_1(\alpha_3+\alpha_2)\overline{c}\\
    \leq& 2T^{\gamma-1}(\alpha_3+\alpha_2)\overline{c} \label{eq:Lemma 2 final}
\end{align}
where {(\ref{eq:Lemma 2 final}) is obtained by using Assumption~\ref{assumption:uniform ergodic}}.
\end{proof}

In the following, we will derive the probabilistic bound on estimating $\mathbb{E}_{\mu_\pi^T}\Delta L(s)$ with $M$ trajectories of $T$ steps. It is worth mentioning that since $M$ trajectories are independent of each other, each trajectory as a whole is applicable for the estimation of $\Delta L(s)$ under $\mu_\pi^T$. This will be demonstrated in the following lemma, where increasing $M$ is desirable for the reduction of estimation deviation, while $T$ does not affect the probabilistic bound.
\begin{lemma}\label{lemma:initial bound}
Let $M$ denote the number of trajectories and $T$ denote the length of trajectories. If there exist positive constants $\alpha _{1}$, $\alpha _{2}$ such that (\ref{Theorem 2-1}) hold, then $\forall\beta\geq0$, the probability that  
\begin{align}\label{eq: lemma 2 eq 1}
    &\frac{1}{MT}\sum_{t=1}^{T}\sum_{m=1}^{M}\Delta L(s_{t,m})
    \leq \mathbb{E}_{\mu_\pi^T} \Delta L(s)
-\beta \vphantom{\frac{1}{MT}\sum_{t=1}^{T}\sum_{m=1}^{M}}
\end{align}
is upper bounded by 
\begin{equation}\label{eq:finite sample bound eq}
    \exp \left(- \frac{2M\beta^2}{(2\alpha_2+\alpha_3)^2\overline{c}^2}\right)
\end{equation}
where 
\begin{equation}
    \Delta L(s_{t,m}) := L(s_{t+1,m})-L(s_{t,m})+\alpha_3c\left(s_{t,m}\right)
\end{equation}
and $s_{t,m}$ denotes the sampled state in the $m$-th trajectory at time $t$.
\end{lemma}

\begin{proof}
The high-level plan is to apply Hoeffding's inequality to establish the probabilistic bound. First, Hoeffding's inequality holds under the condition that random variables follow the same distribution and are independent of each other. {On the right hand side of \eqref{eq: lemma 2 eq 1}, }
\begin{align}
&\mathbb{E}_{\mu_{\pi}^{T}}\Delta L(s)\notag\\
=&\int_{\mathcal{S}}\mu_{\pi}^{T}(s)\Delta L(s)ds\\
=&\int_{\mathcal{S}}\frac{1}{T}\sum\limits_{t=1}^{T}P(s|\pi,\rho,t)\Big( \mathbb{E}_{s'\sim P_\pi}L(s')\notag\\
&-L(s)+\alpha_3c(s)\Big)ds\\
=&\frac{1}{T}\sum\limits_{t=1}^{T}\Bigg(\int_{\mathcal{S}}P(s|\pi,\rho,t+1)L(s)ds\notag\\
&+\int_{\mathcal{S}}P(s|\pi,\rho,t)(-L(s)+\alpha_3c(s))ds\Bigg)
\end{align}
\noindent {Let's denote the probability that the inequality \eqref{eq: lemma 2 eq 1} holds as $\delta$}
\begin{equation}
\begin{aligned} 
\delta := 
    & \mathbb{P}\Bigg( \frac{1}{M}\sum_{m=1}^{M}\sum_{t=1}^{T}\Big(\frac{L(s_{t+1,m})}{T}-\frac{L(s_{t,m})}{T} \\
    &\ \ \ \ \ \ \ +\frac{\alpha_3c(s_{t,m})}{T}\Big) - \mathbb{E}_{\mu_\pi^T} \Delta L(s)\leq -\beta \vphantom{\frac{1}{MOT}\sum_{t=1}^{T}\sum_{m=1}^{M}\sum_{o=1}^{O}} \Bigg)
\end{aligned} 
\end{equation}
which contains a number of $M$ random sequences of variables $L(s_{t+1,m})\in [0,\alpha_2\overline{c}]$ and $\alpha_3c(s_{t,m})-L(s_{t,m})\in [-\alpha_2\overline{c},\alpha_3\overline{c}]$. 
Therefore, by Hoeffding's inequality, the following inequality holds $\forall\beta\geq0$,
\begin{align}
\delta &\leq \exp \left(- \frac{2M^2\beta^2}{M(2\alpha_2+\alpha_3)^2\overline{c}^2}\right)\\
&=\exp \left(- \frac{2M\beta^2}{(2\alpha_2+\alpha_3)^2\overline{c}^2}\right)
\end{align}
which concludes the proof. 
\end{proof}

{In Lemma 3, $\beta$ is a positive constant that quantifies the deviation between the sample average of $\Delta L(s)$ and its expectation. As $\beta$ increases, the probability of such a deviation occurring becomes smaller, which is reflected by the decreasing upper bound.}
Now, the finite sample estimation of $\Delta L(s)$ and $\mathbb{E}_{\mu_\pi }\Delta L(s)$ are connected with $\mathbb{E}_{\mu_\pi^T}\Delta L(s)$ respectively by Lemmas~\ref{lemma:finite time estimation bound} and \ref{lemma:initial bound}, we will unify them to derive the desired probabilistic stability guarantee.

\begin{theorem}\label{theorem:finite sample bound}
{Assume that Assumptions~\ref{assumption: stationary distribution assumption}-\ref{assumption:uniform ergodic} hold for system \eqref{intro:system}.} If there exists a function $L:\mathcal{S}\rightarrow \mathbb{R}_{+}$\ and positive constants $\alpha _{1}$ and $\alpha _{2}$, such that (\ref{Theorem 2-1}) {holds}, and for a number of $M$ trajectories with $T$ time steps there exists a positive constants $\epsilon$ and $\alpha_{3}$ such that 
\begin{equation}
T\geq \left(\frac{b_1}{\epsilon}\right)^{\frac{1}{1-\gamma}}
\label{eq:minimum T}
\end{equation}
and
\begin{equation}\label{eq:finite sample constraint}
    \frac{1}{MT}\sum_{t=1}^{T}\sum_{m=1}^{M}\Delta L(s_{t,m})\leq -\epsilon,
\end{equation}
then the system \eqref{intro:system} controled by $\pi$ is mean square stable with a probability of at least
\begin{equation}
1-\exp \left(-2M \left(\frac{\epsilon-2T^{\gamma-1}b_1}{b_2}\right)^2\right)
 \label{stability bound}
 \end{equation}
where $b_1=(\alpha_3+\alpha_2)\overline{c}$ and $b_2=(2\alpha_2+\alpha_3)\overline{c}$.
If the desired confidence of stability guarantee is at least $\delta$, the associated overall sample complexity is at least $\mathcal{O}(\log(\frac{1}{1-\delta}))$.To achieve a confidence $\delta$, $M$ and $T$ have to satisfy $M (\epsilon-T^{\gamma-1}b_1)^2\geq 2\overline{c}^2(2\alpha_2+\alpha_3)^2\log(\frac{1}{1-\delta})$.
\end{theorem}

\begin{proof}
We will use $\Delta L_{M,T}(s)$ to represent the sample average on the left-hand side of (\ref{eq:finite sample constraint}) for improved readability. Essentially, $\Delta L_{M,T}(s)$ is the unbiased estimation of $\mathbb{E}_{\mu_\pi^T}\Delta L(s)$, of which the probabilistic bound is derived in Lemma~\ref{lemma:initial bound}. However, in Lemma~\ref{lemma:finite time estimation bound} we show that $\mathbb{E}_{\mu_\pi^T}\Delta L(s)$ is not an unbiased estimation of $\mathbb{E}_{\mu_\pi}\Delta L(s)$ since $T$ is finite.  Now based on these lemmas, we need to investigate the effect of estimating $\mathbb{E}_{\mu_\pi}\Delta L(s)$ in (\ref{Theorem 2-2}) with $\Delta L_{M,T}(s)$.

Let $\omega$ denote the bound in (\ref{eq:delta L bound}), $\omega=2\overline{c}(\alpha_3+\alpha_2)T^{\gamma-1}$. According to (\ref{eq:finite sample constraint}), $L_{M,T}(s)+\epsilon$ is semi-negative definite, then
\begin{align}
    & \mathbb{P}(\mathbb{E}_{\mu_\pi} \Delta L(s)\leq0)\notag
    \\
    \geq& \mathbb{P}\left(\mathbb{E}_{\mu_\pi} \Delta L(s)<\Delta L_{M,T}(s)+\epsilon\right)\\
    =&\mathbb{P}\left( \left(\Delta L_{M,T}(s)-\mathbb{E}_{\mu_\pi^T}\Delta L(s) \right) \right. \notag\\
    & \left. \ \ + \left(\mathbb{E}_{\mu_\pi^T}\Delta L(s)-\mathbb{E}_{\mu_\pi} \Delta L(s)\right) >-\epsilon\right)\\
    \geq&\mathbb{P}\left(\Delta L_{M,T}(s)-\mathbb{E}_{\mu_\pi^T}\Delta L(s) - \omega>-\epsilon\right)\\
    =&1- \mathbb{P}\left(\Delta L_{M,T}(s)-\mathbb{E}_{\mu_\pi^T}\Delta L(s)\leq-(\epsilon-\omega) \right)
\end{align}
Then we apply the probabilistic bound derived in Lemma~\ref{lemma:initial bound},
by substituting $\beta$ in (\ref{eq:finite sample bound eq}) by 
\begin{equation}\label{eq: proof 1-1}
(\epsilon-\omega) =  \epsilon-\frac{1}{T}\sum_{t=1}^T\left \Vert\mu_\pi(s)- P(s|\rho,\pi,t)\right\Vert_1b_1
\end{equation}
It follows that
\begin{align}
&\mathbb{P}(\mathbb{E}_{\mu_{\pi}}\Delta L(s)\leq 0)\notag\\
\geq&1-\mathbb{P}(\Delta L_{M,T}(s)-\mathbb{E}_{\mu_\pi^T}\Delta L(s)\leq-(\epsilon-\omega))\\
\geq&1-\exp \left(- \frac{2M(\epsilon-\omega)^2}{(2\alpha_2+\alpha_3)^2\overline{c}^2}\right)\label{eq:proof 1-2}
\end{align}
Substitute \eqref{eq: proof 1-1} into \eqref{eq:proof 1-2} and utilize Assumption~\ref{assumption:uniform ergodic}, one has
\begin{equation}
\begin{aligned}
&\mathbb{P}(\mathbb{E}_{\mu_{\pi}}\Delta L(s)\leq 0)\\
\geq &1-\exp \left(- \frac{2M(\epsilon-T^{\gamma-1}2b_1)^2}{(2\alpha_2+\alpha_3)^2\overline{c}^2}\right)
\end{aligned}
\end{equation}

The desired bound for guaranteeing mean square stability in probability can be obtained in \eqref{stability bound}.
\end{proof}

\section{Reinforcement learning with probabilistic stability guarantee}\label{sec:6}

Based on the theoretical results in the previous section, one can judge whether the system is stable given several finite-length trajectories by evaluating (\ref{eq:finite sample constraint}). The theoretical results in the stability theorems using Lyapunov's method do not, however, give a prescription for determining the Lyapunov function and controller. To translate the theorem into practical algorithms, the high-level plan is to parameterize $L(s)$ with (\ref{eq:lyapunov network}) and the controller $\pi(a|s)$ with an arbitrary NN $\pi_{\theta}(a|s)$. Then $\phi$ and $\theta$ will be updated separately and iteratively using stochastic gradient descent algorithms until system~(\ref{intro:system}) is stabilized such that (\ref{eq:finite sample constraint}) is satisfied. We use $\tau$ to
denote a trajectory ($\tau = \{ s_1,a_1,s_2...s_T\}$), and $\tau \sim \pi$ is the shorthand for indicating that the distribution over trajectories depends on $\pi$: $P(\tau) = \rho(s_1)\prod^T_{t=1}P(s_{t+1}|s_t,a_t)\pi(a_t|s_t)$.

\subsection{Policy gradient}
\label{sec:Policy Gradient}

In this subsection, we will focus on how to learn the controller in an iterative manner, repeatedly estimating the policy gradient of the target function with samples and updating $\theta$ through stochastic gradient descent. $\Delta L(s)$ is temporarily assumed to be given, i.e., $\phi$ are fixed.  
In Section~\ref{sec:Lyapunov Function}, we will show how the Lyapunov function is selected and learned after $\theta$ is updated.

Since the left-hand side of (\ref{eq:finite sample constraint}) is the unbiased estimate of $\Delta L(s)$ on $\mu_\pi^T$, the problem can be formulated by
\begin{align}
\begin{aligned}
    \text{find } &\theta
    \text{, s.t. } \mathbb{E}_{\mu_{\pi_\theta}^T}\Delta L(s)\leq  -\epsilon 
\end{aligned}\label{eq:target function}
\end{align}
A straightforward way of solving the constrained optimization problem above would be the first-order method \cite{bertsekas2014constrained} (Chapter 4), also known as gradient descent. At each update step, the gradient of (\ref{eq:target function}) with respect to $\theta$ is estimated with samples, and $\theta$ updates a small step in the opposite direction of the estimated gradient vector.
The gradient of (\ref{eq:target function}) with respect to $\theta$ is derived in the following theorem.

\begin{theorem}\label{theorem:policy gradient}
The gradient of Lyapunov condition (\ref{eq:target function}) is given by the following,

\begin{equation}
\begin{aligned}
&\tauE{\pi_\theta}{\frac{1}{T}\sum_{t=1}^{T}\nabla_\theta\log\pi_\theta(a_t|s_t){l(\tau, t)}} 
\end{aligned} \label{eq:policy gradient}
\end{equation}
{where
\begin{equation}
l(\tau, t) = \sum_{j=t+1}^{T} c(s_j) +L(s_{T+1})
\end{equation}}
\end{theorem}

\begin{proof}
In this section, we will show that (\ref{eq:policy gradient}) holds.
To proceed, we introduce $P(s_t|\rho,\pi_\theta):= P(s|\rho,\pi_\theta,t)$ to make the derivation more clear. First, the following equalities hold,
\begin{equation}
\begin{aligned}
    &\mathbb{E}_{\mu_{\pi_\theta}^T}\Delta L(s)+\epsilon\\
    =&\frac{1}{T}\mathbb{E}_{P(\cdot|\rho,\pi_\theta)}L(s_{T+1}) -\frac{1}{T}\mathbb{E}_{\rho}L(s_1)\\
    &+ \frac{\alpha_3}{T}\sum_{t=1}^T\mathbb{E}_{P(\cdot|\rho,\pi_\theta)}c(s_t)+\epsilon
\end{aligned}
\end{equation}

Take the derivative of the above equation with respect to $\theta$, one has 
\begin{equation}\label{proof:gradient-1}
\begin{aligned}
&\nabla_\theta\mathbb{E}_{\mu_{\pi_\theta}^T}\Delta L(s)\\ 
=& \frac{1}{T}\int_{\mathcal{S}}\nabla_\theta P(s_{T+1}|\rho,\pi_\theta)L(s_{T+1})  \\ &+\frac{\alpha_3}{T}\sum_{t=1}^T\int_{\mathcal{S}}\nabla_\theta P(s_t|\rho,\pi_\theta)c(s_t)
\end{aligned}
\end{equation}

For any $t\in\mathbb{Z}_{[2,+\infty)}$ and function $f(s_t)$,
\begin{align}
&\int_{\mathcal{S}}\nabla_\theta P(s_t|\rho,\pi_\theta)f(s_t) \notag \\
=&\mathbb{E}_{ P(s_{t-1}|\rho,\pi_\theta)}\Bigg(\sum_{a_{t-1}} \nabla_\theta\pi_\theta(a_{t-1}|s_{t-1})\notag\\&\times\int_{\mathcal{S}}P(s_t|s_{t-1},a_{t-1})f(s_t)  \notag\\
&+ \int_{\mathcal{S}}\nabla_\theta P(s_{t-1}|\rho,\pi_\theta)\mathbb{E}_{P_{\pi_\theta}(s_t|s_{t-1})}f(s_t)\Bigg)\\
=&\tauE{\pi_\theta}{f(s_t)\nabla_\theta\log\pi_\theta(a_{t-1}|s_{t-1})}\notag\\
&+\int_{\mathcal{S}}\nabla_\theta P(s_{t-1}|\rho,\pi_\theta)\mathbb{E}_{P_{\pi_\theta}(s_t|s_{t-1})}f(s_t)
\end{align}
Iterate the above relation to $t=1$, one has
\begin{equation}
\begin{aligned}
    &\int_{\mathcal{S}}\nabla_\theta P(s_t|\rho,\pi_\theta)f(s_t)\\ 
    =&\tauE{\pi_\theta}{f(s_t)\sum_{i=1}^{t-1} \nabla_\theta\log\pi_\theta(a_{i}|s_{i})}
\end{aligned}
\end{equation}
Substitute the above equation into (\ref{proof:gradient-1}), 
\begin{equation}
\begin{aligned}
&\nabla_\theta\mathbb{E}_{\mu_{\pi_\theta}^T}\Delta L(s)\\ 
=& \frac{1}{T}\tauE{\pi_\theta}{\sum_{t=1}^{T} \nabla_\theta\log\pi_\theta(a_{t}|s_{t})L(s_{T+1})}\\&+ \frac{\alpha_3}{T}\sum_{t=1}^{T}\tauE{\pi_\theta}{c(s_t)\sum_{t'=1}^{t-1}\nabla_\theta\log\pi_\theta(a_{t'}|s_{t'})}\label{proof:gradient-2}
\end{aligned}
\end{equation}

For the term in (\ref{proof:gradient-2}), we move the first sum into the expectation and exchange its order with the second sum, then
\begin{align}
    &\nabla_\theta\mathbb{E}_{\mu_{\pi_\theta}^T}\Delta L(s)\notag\\ 
=&\frac{1}{T}\tauE{\pi_\theta}{\sum_{t=1}^{T} \nabla_\theta\log\pi_\theta(a_{t}|s_{t})L(s_{T+1})} \notag\\
&+ \frac{\alpha_3}{T}\tauE{\pi_\theta}{\sum_{t'=1}^{T-1}\nabla_\theta\log\pi_\theta(a_{t'}|s_{t'})\sum_{t=t'+1}^{T}c(s_t)}\\
=&\frac{1}{T}\tauE{\pi_\theta}{\sum_{t=1}^{T} \nabla_\theta\log\pi_\theta(a_{t}|s_{t})l(\tau, t)}
\end{align}
which concludes the proof.
\end{proof}

Surprisingly, we found that the policy gradient derived in Theorem~\ref{theorem:policy gradient} is very similar to that used in the vanilla policy gradient method, i.e., REINFORCE \cite{sutton2018reinforcement}, in the classic RL paradigm. In RL, the objective is to minimize a certain objective function $J_\theta = \tauE{\pi_\theta}{ \sum_{t=1}^{T} c(s_t)}\label{eq:RL objective}$ and the policy gradient $\nabla_\theta J_{\theta}$ is given as follows:
\begin{align}
\tauE{\pi_\theta}{\sum_{t=1}^{T}\nabla_\theta\log\pi_\theta(a_t|s_t)\left(\sum_{j=t+1}^{T} c(s_j)\right)}
\label{eq:REINFORCE policy gradient}
\end{align}

Essentially, despite the scale of $\frac{1}{T}$, (\ref{eq:policy gradient}) and (\ref{eq:REINFORCE policy gradient}) are equivalent if one chooses $c(s)$ to be the Lyapunov function and sets $\alpha_3 =1$.
This implies that given the system (\ref{intro:system}), REINFORCE actually updates the policy towards a solution that can stabilize the system, although it is now aware of under what conditions the solution is guaranteed to be stabilizing. In particular, we can view REINFORCE as a special case of our result, since we prove that many other choices of $\alpha_3$ and Lyapunov functions are admissible to find a stabilizing solution. The default setting of $c(s)$ as $L(s)$ and $\alpha_3=1$ in REINFORCE may not satisfy (\ref{eq:finite sample constraint}), while we revealed that many other feasible combinations of $L$ and $\alpha_3$ potentially exist.

In light of this connection with REINFORCE, it is natural to propose a similar learning procedure based on Theorem~\ref{theorem:policy gradient}, referred to as Lyapunov-based REINFORCE (L-REINFORCE). L-REINFORCE updates the policy with the policy gradient proposed in (\ref{eq:policy gradient}). Instead of minimizing the sum of costs, L-REINFORCE aims to learn a stochastic policy $\pi_\theta(a|s)$ such that the conditions in Theorem~\ref{theorem:finite sample bound} are satisfied.

Furthermore, to reduce the variance in the estimation of (\ref{eq:policy gradient}) and speed up learning, it is desirable to introduce a baseline function $b(s)$ in (\ref{eq:policy gradient}), and the estimation is still unbiased \cite{sutton2018reinforcement}:
\begin{align}
&\tauE{\pi_\theta}{\frac{1}{T}\sum_{t=1}^{T}\nabla_\theta\log\pi_\theta(a_t|s_t) \Big({l(\tau, t)}-b(s_t)\Big)} \label{eq:policy gradient with baseline}
\end{align}

\subsection{Lyapunov function}
\label{sec:Lyapunov Function}
The Lyapunov function is parameterized using a DNN $f_\phi$ in (\ref{eq:lyapunov network}). 
In this paper, we choose the value function to be the update target for $f$ as explored in \cite{chow2019lyapunov} and leave other possible choices for future work. 

In L-REINFORCE, there are two networks: the policy network $\mu_\theta(s)$ and the Lyapunov network $f_\phi(s)$. For the policy network, we use a fully connected MLP with one hidden layer of 32 units activated by ReLU, outputting the logits for softmax. Its expression is as follows, 
\begin{equation}
    \pi(a|s)=\frac{exp(\mu_\theta^a(s))}{\sum_iexp(\mu_\theta^i(s))}
\end{equation}
where $\mu_\theta^i(s)$ denotes the $i_\text{th}$ dimension of $\mu_\theta(s)$.
$f_\phi(s)$ is a fully connected MLP with two hidden layers activated by ReLU. $f$ is trained to approximate the value function, $V(s_t)= \sum_{j=t}^\infty \gamma^{j-t} \mathbb{E}_{s_j}\left(c(s_{j})-\overline{b}\right)$, where $\overline{b}$ is a designed positive constant. 
The parameter of $f$ is updated through soft replacement, i.e. $\phi_{k+1} \leftarrow (1-\tau)\phi_k + \tau (\phi_k - \alpha\nabla_\phi e^2)$, where $e$ is the temporal difference. 

To wrap up, the L-REINFORCE algorithm is summarized in Algorithm~\ref{algo:L-REINFORCE}. 

\begin{algorithm}[H]
   \caption{L-REINFORCE}
   \label{algo:L-REINFORCE}
\begin{algorithmic}
   \REPEAT
   \FOR{1, 2, \dots, M}
   \STATE Collect transition pairs following $\pi_\theta$ for $T$ steps
   \ENDFOR
   \STATE 
   $
   \theta \leftarrow \theta - \alpha \nabla_\theta\mathbb{E}_{\mu_{\pi_\theta}^T}\Delta L(s)
   $
   \STATE Update $\phi$ of Lyapunov function/value network to approximate the designed target
   \UNTIL{There exists $\alpha_3$ such that $\mathbb{E}_{\mu_{\pi_\theta}^T}\Delta L(s)<-\epsilon$}
\end{algorithmic}

\end{algorithm}

\begin{figure}[htb] 
\centering
\includegraphics[width=0.48\textwidth]{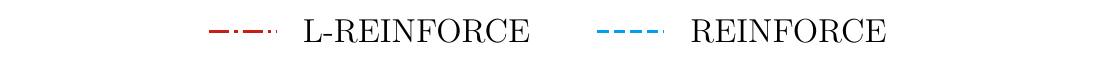}
\centering
\includegraphics[width=0.48\textwidth]{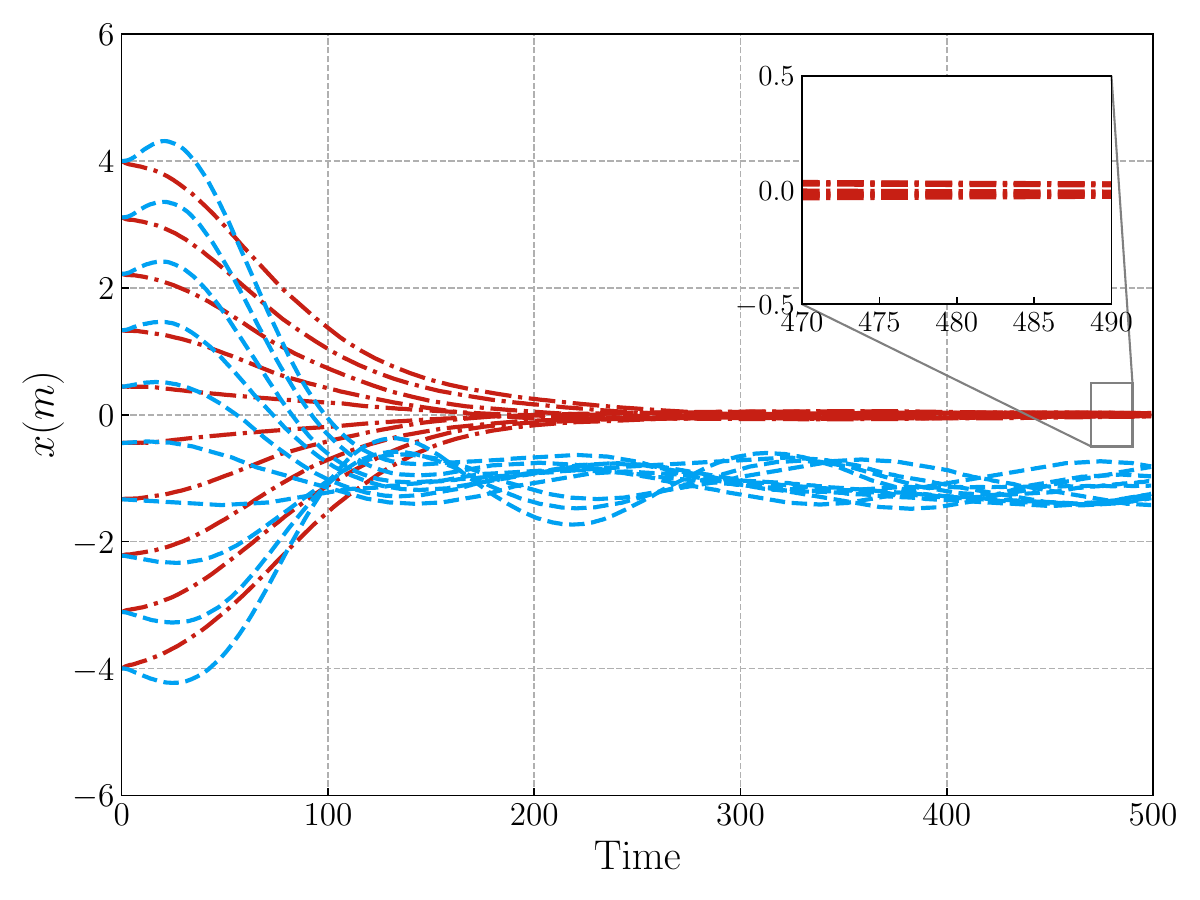}
\centering
\includegraphics[width=0.48\textwidth]{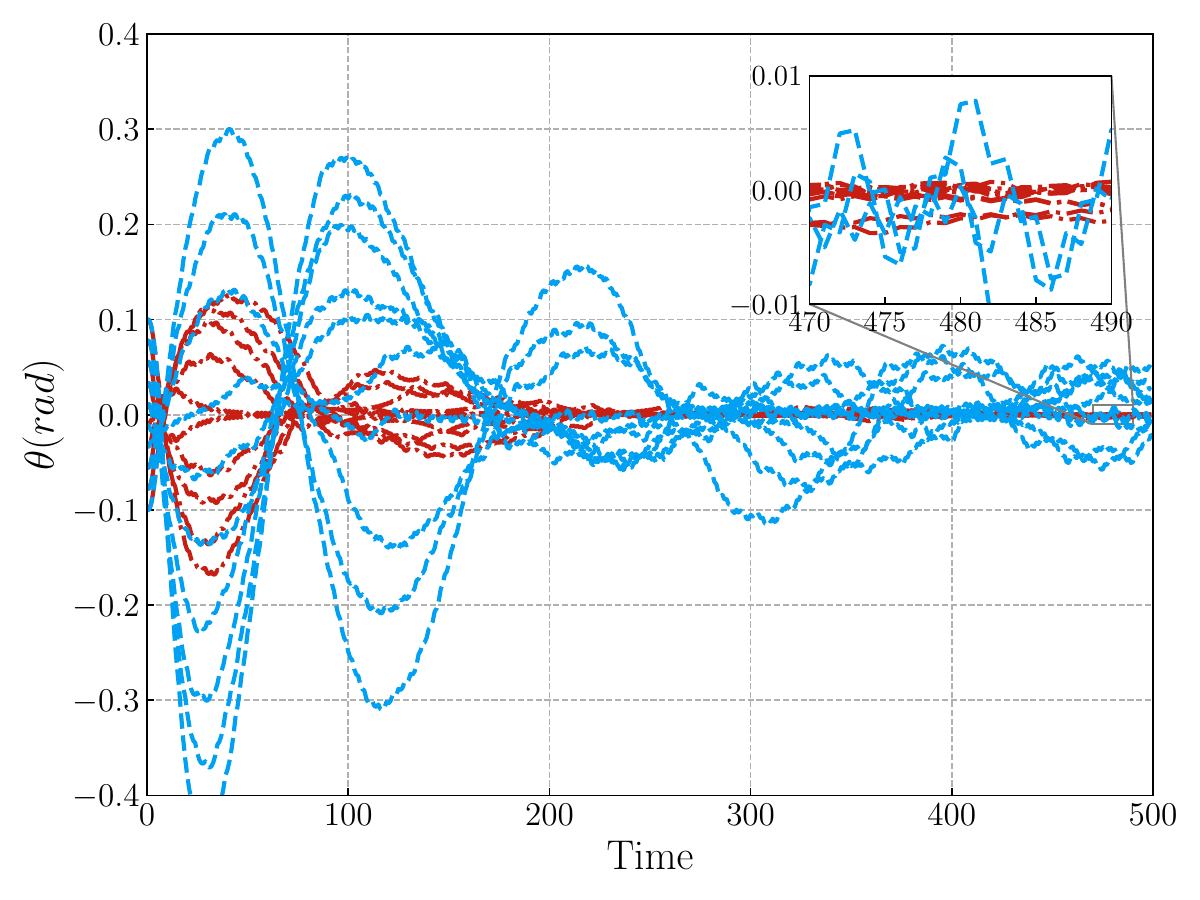}

\caption{{State trajectories of the controllers trained by L-REINFORCE and REINFORCE. The X-axis denotes the time steps; and the Y-axis denotes the position $x$ in meters and the angle $\theta$ in radians, respectively. Zoom-in views are displayed inside the plots.} }
\label{fig:phase trajectory}
\end{figure}

\section{Simulation results}\label{sec:experiment}

In this section, we evaluate the performance of the proposed framework. 
To demonstrate the effectiveness of the proposed method, we consider the stabilization task of a simulated Cartpole \cite{brockman2016openai}. The goal is to stabilize the pole vertically at the position $x=0$. We adopt REINFORCE \cite{williams1992simple} as the baseline method for comparison. L-REINFORCE and REINFORCE select the action, the horizontal force $F\in\{-10, 0, 10\}(N)$ applied to the cart.
The hyperparameters of L-REINFORCE are presented in Table~\ref{table:hyperparameters}. { In Table~\ref{table:hyperparameters}, the structures of the neural networks in L-REINFORCE are represented by tuples. $(64, 16, 1)$ represents that the Lyapunov network is composed of three consecutive layers, which contain 64 neurons, 16 neurons, and 1 neuron, respectively. The output of this network is a scalar term.  $(32, 3)$ represents that the policy network is composed of two consecutive layers, which contain 32 neurons and 3 neurons, respectively. }

The initial horizontal position of the cart $x$ is uniformly distributed on $[-1,1](m)$, and other states are uniformly distributed on $[-0.05,0.05]$. {During training, the episode ends if the angle $\theta$ exceeds the thresholds $\pm 0.35 \text{ rads}$ or the total time step exceeds $500$. This condition is not applied during the post-training performance evaluation.} {During training, REINFORCE uses the 2-norm of the state as the stage cost. The learning rate of REINFORCE is determined to be $1\times10^{-2}$ after tuning. During training, REINFORCE and L-REINFORCE adopt the same values for the number of trajectories $M$ and time steps $T$ for each iteration of policy update. }

It is important to note that the stability of a system can not be judged from the cumulative cost (or return) because stability is an inherent property of the system dynamics, and a stable system may not be optimal in terms of the return.
Therefore in Fig.~\ref{fig:phase trajectory}, we show the transient system behavior under the learned policies of L-REINFORCE and REINFORCE. The systems are initialized at different positions in the space, and their subsequent behaviors are observed.
As shown in Fig.~\ref{fig:phase trajectory}, starting from different initial states, L-REINFORCE can efficiently stabilize the system.  {On the other hand, the cartpole controlled by REINFORCE is not stabilized in the x position and oscillates in the angular position. }

To let the readers have an intuitive sense of the probabilistic stability bound, the bound for Cartpole under the control of the L-REINFORCE agent is shown in Fig.~\ref{fig:sharp-1}. As shown in Fig.~\ref{fig:sharp-1}, the probability of stability increases sharply as the minimum $T$ requirement (\ref{eq:minimum T}) is satisfied. Increasing $M$ and $T$ are both helpful for raising the confidence of stability guarantee. The probabilistic stability guarantee measures the reliability of learned policies. By tuning the hyperparameters such as $M$, $T$, and $\alpha_3$, one can achieve the confidence of stability guarantee according to the real needs.

\begin{figure}[htb] 

\centering

\includegraphics[width=0.48\textwidth]{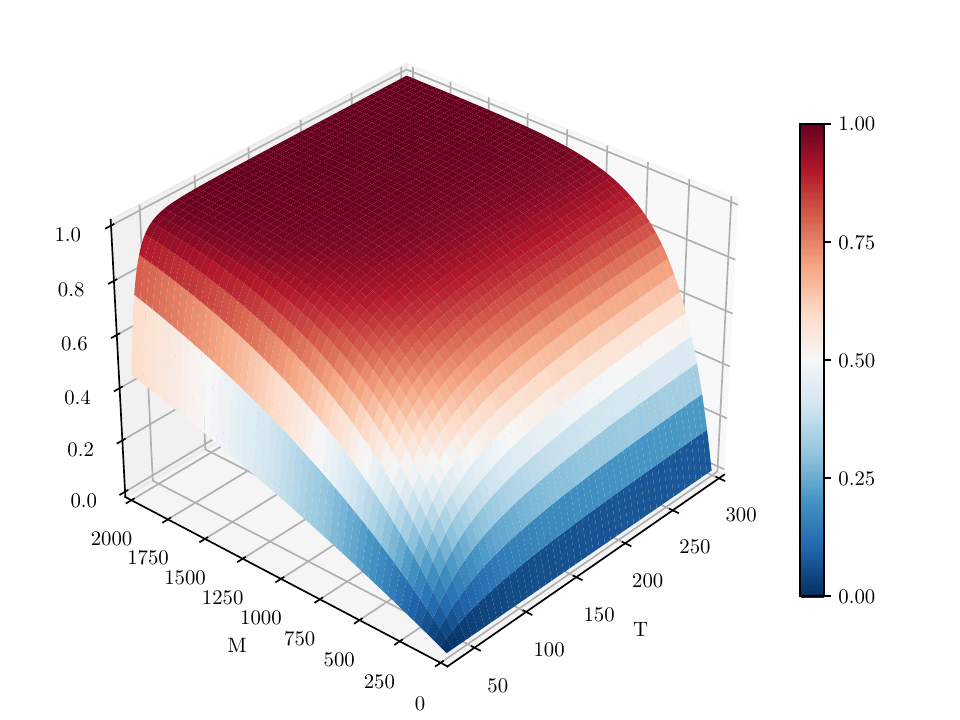}

\caption{Visualization of the probabilistic stability bound. The X-axis indicates the length of trajectories $T$ and Y-axis indicates the number of episodes $M$. The Z-axis indicates the probability of stability and the values are colored differently according to the color bar. }
\label{fig:sharp-1}
\end{figure}

{
\begin{remark}
    The choice of the clip value $\bar{c}$ may influence both controller performance and sample complexity.  
    Specifically, 1) when $\Vert s\Vert>\bar{c}$, the cost saturates and provides limited gradient information for training. A larger $\bar{c}$ gives the controller more informative feedback in these regions. 2) On the other hand, a larger $\bar{c}$ loosens the confidence bound in \eqref{stability bound}. Consequently, more data is required to achieve the same confidence level. Therefore, $\bar{c}$ indeed acts as a hyperparameter, and its choice reflects a trade-off between informative controller training and the tightness of the probabilistic stability bound. 
\end{remark}
}
\section{Conclusion and discussion}

In this paper, we have investigated the finite sample-based stability analysis and control of stochastic systems characterized by MDP in a model-free manner. Instead of verifying energy decreasing point-wisely on the state space, we proposed a stability theorem where only one sampling-based inequality is to be checked. Furthermore, we showed that with a finite number of trajectories of finite length, it is possible to guarantee stability with a certain probability, and the probabilistic bound is derived. Finally, we proposed a model-free learning algorithm to learn the controller with a stability guarantee and revealed its connection to REINFORCE. 

In the future, an important direction is to extend the theoretical analysis to more efficient algorithms than REINFORCE. 
{
On the other hand, the notion of stability in Definition 1 follows the description of MSS given in \cite{bolzern2010markov,zhu2015mean,benjelloun2002mean}. In \cite{teel2014stability}, a variety of definitions of stability in stochastic hybrid systems have been reviewed, including exponential stability, Lyapunov stability, and Lagrange stability, etc. 
We believe the extension of the proposed method to investigate other types of stability is another promising direction for future research. 
}

\begin{table}[htb]
\caption{Hyperparameters of L-REINFORCE}\label{table:hyperparameters}
\begin{center}
\begin{tabular}{l|c }
\hline
Hyperparameters&Value\\
\hline
Number of trajectories ($M$)& 20\\
Time steps ($T$)& 250\\
learning rate ($\alpha$) & 1e-2\\
Soft replacement ($\tau$) &0.005\\
$\alpha_3$ in (\ref{eq:finite sample constraint}) &0.005\\
$\epsilon$ in (\ref{eq:finite sample constraint}) & 1e-4\\
$\sigma$ in (\ref{eq:lyapunov network}) & 0.01\\
Upper bound of norm ($\overline{c}$) & 1000 \\
Discount($\gamma$) & 0.995 \\
Bias ($\overline{b}$) & 10 \\
Structure of Lyapunov network {$f_\phi(s)$} &(64,16,1)\\
Structure of policy network {$\mu_\theta(s)$} &(32,3)\\
\hline
\end{tabular}
\end{center}

\end{table}


\end{document}